\documentclass[12pt,twoside]{article}
\usepackage{latexsym,amsmath,amssymb}
\usepackage{wrapfig}
\usepackage{tikz}

\newcommand{\eqn}[1]{\begin{align}#1\end{align}}
\newcommand{\eq}[1]{\begin{align*}#1\end{align*}}
\newcommand{\proofof}{Proof of }
\newcommand{\proofend}{}

\def\subsubsect#1{\vspace{1ex plus 0.5ex minus 0.5ex}\noindent{\bf\boldmath{#1.}}}

\usetikzlibrary{trees,arrows,automata}
\renewcommand{\v}[1]{{\boldsymbol #1}}
\newcommand{\argmax}{\operatornamewithlimits{arg\,max}}
\newcommand{\ind}[1]{[\![ #1 ]\!]}
\newcommand{\R}[0]{\mathbb R}
\newcommand{\N}[0]{\mathbb N}

\usepackage{amsthm}
\theoremstyle{plain}
\newtheorem{theorem}{Theorem}
\newtheorem{lemma}[theorem]{Lemma}
\theoremstyle{definition}
\newtheorem{definition}[theorem]{Definition}
\newtheorem{assumption}[theorem]{Assumption}
\theoremstyle{remark}

\newenvironment{keywords}{\centerline{\bf\small
Keywords}\begin{quote}\small}{\par\end{quote}\vskip 1ex}

\newcommand{\yx}{y\!x}
\newcommand{\yr}{y\!r}
\newcommand{\A}{\mathcal Y}
\newcommand{\Rc}{\mathcal R}
\newcommand{\Mc}{\mathcal M}
\newcommand{\X}{\mathcal X}
\newcommand{\B}{\mathcal B}
\newcommand{\E}{\mathbf E}
\renewcommand{\Alph}{\mathcal A}
\renewcommand{\d}{\gamma}
\renewcommand{\O}{\mathcal O}

\begin{document}

\title{
\vskip 2mm\bf\Large\hrule height5pt \vskip 4mm
Asymptotically Optimal Agents
\vskip 4mm \hrule height2pt}
\author{{\bf Tor Lattimore}$^1$ and {\bf Marcus Hutter}$^{1,2}$\\[3mm]
\normalsize Research School of Computer Science \\[-0.5ex]
\normalsize $^1$Australian National University and $^2$ETH Z{\"u}rich \\[-0.5ex]
\normalsize\texttt{\{tor.lattimore,marcus.hutter\}@anu.edu.au}
}
\date{25 July 2011}

\maketitle

\begin{abstract}
Artificial general intelligence aims to create agents capable of learning to solve arbitrary interesting problems. We define two versions of
asymptotic optimality and prove that no agent can satisfy the strong version while in some cases, depending on discounting, there
does exist a non-computable weak asymptotically optimal agent.
\def\contentsname{\centering\normalsize Contents}
{\parskip=-2.7ex\tableofcontents}
\end{abstract}


\newpage
\section{Introduction}

The dream of artificial general intelligence is to create an agent that, starting with no knowledge of its environment, eventually
learns to behave optimally. This means it should be able to learn chess just by playing, or Go, or how to drive a car or mow the lawn, or any
task we could conceivably be interested in assigning it.

Before considering the existence of universally intelligent agents, we must be precise about what is meant by optimality.
If the environment and goal are known, then subject to computation issues, the optimal policy is easy to construct using an expectimax search
from sequential decision theory \cite{NR03}.
However, if the true environment is unknown then the agent will necessarily spend some time exploring, and so
cannot immediately play according to the optimal policy.
Given a class of environments, we suggest two definitions of asymptotic optimality for an agent.
\begin{enumerate}
\item An agent is strongly asymptotically optimal if for every environment in the class it plays optimally in the limit.
\item It is weakly asymptotic optimal if for every environment in the class it plays optimally {\it on average} in the limit.
\end{enumerate}
The key difference is that a strong asymptotically optimal agent must eventually stop exploring, while a weak asymptotically optimal agent
may explore forever, but with decreasing frequency.

In this paper we consider the (non-)existence of weak/strong asymptotically optimal agents in the class of all deterministic computable environments.
The restriction to deterministic is for the sake of simplicity and because the results for this case are already sufficiently non-trivial to be interesting.
The restriction to computable is more philosophical. The Church-Turing thesis is the unprovable hypothesis that anything that can intuitively be
computed can also be computed by a Turing machine. Applying this to physics leads to the strong Church-Turing thesis that
the universe is computable (possibly stochastically computable, i.e.\ computable when given access to an oracle of random noise). Having made these assumptions, the largest interesting class then becomes the class of computable
(possibly stochastic) environments.

In \cite{Hut04}, Hutter conjectured that his universal Bayesian agent, AIXI, was weakly asymptotically optimal in the class
of all computable stochastic environments. Unfortunately this
was recently shown to be false in \cite{Ors10}, where it is proven that no Bayesian agent (with a static prior) can be
weakly asymptotically optimal in this class.\footnote{Or even the class of computable deterministic environments.} The key idea behind Orseau's proof was to show that AIXI eventually stops exploring. This is somewhat surprising because it is normally assumed that Bayesian agents solve
the exploration/exploitation dilemma in a principled way. This result is a bit reminiscent of Bayesian (passive induction)
inconsistency results \cite{DF86a,DF86b}, although the details of the failure are very different.

We extend the work of \cite{Ors10}, where only Bayesian agents are considered, to show that non-computable weak asymptotically optimal
agents do exist in the class of deterministic computable environments for some discount
functions (including geometric), but not for others. We also show that no asymptotically optimal agent can be computable, and that
for all ``reasonable'' discount functions there does not exist a strong asymptotically optimal agent.

The weak asymptotically optimal agent we construct is similar to AIXI, but with an exploration component similar to $\epsilon$-learning for finite
state Markov decision processes or the UCB algorithm for bandits. The key is to explore sufficiently often and deeply to ensure that the environment
used for the model is an adequate approximation of the true environment. At the same time, the agent must explore infrequently enough that it actually
exploits its knowledge. Whether or not it is possible to get this balance right depends, somewhat surprisingly, on how forward looking the agent
is (determined by the discount function). That it is sometimes not possible to explore enough to learn the true environment without damaging even
a weak form of asymptotic optimality is surprising and unexpected.

Note that the exploration/exploitation problem is
well-understood in the Bandit case \cite{Auer:02ucb,Berry:85}
and for (finite-state stationary) Markov decision
processes \cite{Strehl:08}. In these restrictive settings,
various satisfactory optimality criteria are available.
In this work, we do not make any assumptions like Markov,
stationary, ergodicity, or else besides computability of the
environment. So far, no satisfactory optimality definition is
available for this general case.

\section{Notation and Definitions}

We use similar notation to \cite{Hut04,Ors10} where the agent takes actions and the environment returns an observation/reward pair.

\subsubsect{Strings}
A finite string $a$ over alphabet $\Alph$ is a finite sequence $a_1a_2a_3\cdots a_{n-1}a_n$ with $a_i \in \Alph$.
An infinite string $\omega$ over alphabet $\Alph$ is an infinite sequence $\omega_1\omega_2\omega_3\cdots$.
$\Alph^n$, $\Alph^*$ and $\Alph^\infty$ are the sets of strings of length $n$, strings of finite length, and infinite strings
respectively.
Let $x$ be a string (finite or infinite) then substrings are denoted
$x_{s:t} := x_s x_{s+1} \cdots x_{t-1} x_{t}$ where $s,t\in\N$ and  $s \leq t$.
Strings may be concatenated.
Let $x,y \in \Alph^*$ of length $n$ and $m$ respectively, and $\omega \in \Alph^\infty$. Then define
$xy := x_1 x_2 \cdots x_{n-1} x_{n} y_1 y_2 \cdots y_{m-1} y_m$ and
$x\omega := x_1 x_2 \cdots x_{n-1} x_{n} \omega_1 \omega_2 \omega_3 \cdots$.
Some useful shorthands,
\eqn{
\label{eqn-string} x_{<t} &:= x_{1:t-1} & \yx_{<t} := y_1 x_1 y_2 x_2 \cdots y_{t-1} x_{t-1}.
}
The second of these is ambiguous with concatenation, so wherever $\yx_{<t}$ appears we assume the interleaving
definition of (\ref{eqn-string}) is intended. For example, it will be common to see $\yx_{<t} y_t$, which represents the string
$y_1 x_1 y_2 x_2 y_3 x_3 \cdots y_{t-1} x_{t-1} y_t$.
For binary strings, we write $\#1(a)$ and $\#0(a)$ to mean the number of 0's and number of 1's in $a$ respectively.

\setlength{\intextsep}{0pt}
\begin{wrapfigure}{r}{5.5cm}
\topsep=0cm
\begin{center}
\tikzstyle{block} = [draw, minimum width=1.2cm, minimum height=0.7cm, scale=0.75]
\begin{tikzpicture}[>=latex', ->, scale=0.75]
\foreach \x in {1,2,3,4,5} {
\node[block] at (1.2*\x, 2.8) (input\x) {$o_\x | r_\x$};
\node[block] at (1.2*\x, 0) (output\x) {$y_\x$};
}
\node[block] at (1.2*6, 0) (input7) {$\cdots$};
\node[block] at (1.2*6, 2.8) (input7) {$\cdots$};
\node[block] at (2.4, 1.4) (agent) {$agent$, $\pi$};
\node[block] at (6.5, 1.4) (env) {$environment$, $\mu$};

\path (input3) edge node {} (agent)
      (agent) edge node {} (output4)
      (output4) edge node {} (env)
      (env) edge node {} (input4)
;
\end{tikzpicture}
\end{center}
\end{wrapfigure}
\subsubsect{Environments and Optimality}
Let $\A$, $\O$ and $\Rc \subset \R$ be action, observation and reward spaces respectively. Let $\X = \O\times\Rc$.
An agent interacts with an environment as illustrated in the diagram on the right. First, the agent takes an action, upon which
it receives a new observation/reward pair. The agent then takes another action, receives another observation/reward pair, and so-on
indefinitely. The goal of the agent is to maximise its discounted rewards over time.
In this paper we consider only deterministic environments where the next observation/reward pair is determined
by a function of the previous actions, observations and rewards.

\begin{definition}[Deterministic Environment]
A {\it deterministic environment} $\mu$ is a function $\mu:(\A\times\X)^*\times \A \to \X$
where $\mu(\yx_{<t}y_t) \in \X$ is the observation/reward pair given after action $y_t$ is taken in history $\yx_{<t}$.
Wherever we write $x_t$ we implicitly assume $x_t = (o_t, r_t)$ and refer to $o_t$ and $r_t$ without defining them.
An environment $\mu$ is computable if there exists a Turing machine that computes it.
\end{definition}
Note that since environments are deterministic the next observation need not depend on the previous observations (only actions). We choose to leave
the dependence as the proofs become clearer when both the action and observation sequence is more visible.

\begin{assumption}\label{assumption1}
$\A$ and $\O$ are finite, $\Rc = [0, 1]$.
\end{assumption}

\begin{definition}[Policy]
A {\it policy} $\pi$ is a function from a history to an action $\pi:(\A\times\X)^* \to \A$.
\end{definition}
As expected, a policy $\pi$ and environment $\mu$ can interact with each other to generate a play-out sequence of
action/reward/observation tuples.
\begin{definition}[Play-out Sequence]
We define the {\it play-out sequence} $\yx^{\mu,\pi} \in (\A\times\X)^\infty$ inductively by
$y^{\mu, \pi}_k := \pi(\yx^{\mu,\pi}_{<k})$ and  $x^{\mu, \pi}_k := \mu(\yx^{\mu,\pi}_{<k}y_k^{\mu,\pi})$.
\end{definition}
We need to define the value of a policy $\pi$ in environment $\mu$. To avoid the possibility of infinite rewards, we
will use discounted values. While it is common to use only geometric discounting, we have two reasons to allow arbitrary
time-consistent discount functions.
\begin{enumerate}
\item Geometric discounting has a constant effective horizon, but we feel agents should be allowed to use a discount function that leads
to a growing horizon. This is seen in other agents, such as humans, who generally become less myopic as they grow older. See \cite{FOO02}
for a overview of generic discounting.
\item The existence of asymptotically optimal agents depends critically on the effective horizon of the discount function.
\end{enumerate}

\begin{definition}[Discount Function] \label{defn_discount}
A regular discount function $\v\d \in \R^\infty$ is a vector satisfying $\gamma_k \geq 0$ and
$0 < \sum_{t=k}^\infty \d_t < \infty$ for all $k \in \N$.
\end{definition}
The first condition is natural for any definition of a discount function. The second condition is often cited
as the purpose of a discount function (to prevent infinite utilities), but economists sometimes use
non-summable discount functions, such as hyperbolic. The second condition also guarantees the agent cares
about the infinite future, and is required to make asymptotic analysis interesting. We only
consider discount functions satisfying all three conditions. In the following, let
\eq{
\Gamma_t &:= \sum_{i=t}^\infty \d_i &
H_t(p) &:= \min_{h\in\N} \left\{h : {1 \over \Gamma_t} \sum_{k=t}^{t + h} \d_k > p\right\}.
}
An infinite sequence of rewards starting at time
$t$, $r_t, r_{t+1}, r_{t+2},\cdots$ is given a value of ${1 \over \Gamma_t} \sum_{i=t}^\infty \d_i r_i$. The term ${1 \over \Gamma_t}$
is a normalisation term to ensure that values scale in such a way that they can still be compared in the limit.
A discount function is computable if there exists a Turing machine computing it. All well known discount functions, such as geometric,
fixed horizon and hyperbolic are computable.
Note that $H_t(p)$ exists for all $p \in [0, 1)$ and represents the effective horizon of the agent. After $H_t(p)$ time-steps into the
future, starting at time $t$, the agent stands to gain/lose at most $1 - p$.

\begin{definition}[Values and Optimal Policy]\label{defn_optimal}
The value of policy $\pi$ when starting from history $\yx^{\mu,\pi}_{<t}$ in environment $\mu$
is
$V_\mu^\pi(\yx^{\mu,\pi}_{<t}) := {1 \over \Gamma_t} \sum_{k = t}^\infty \gamma_k r^{\mu,\pi}_k$.
The optimal policy $\pi^*_{\mu}$ and its value $V^*_\mu$ are defined
$\pi^*_{\mu}(\yx_{<t}) := \argmax_\pi V_\mu^\pi(\yx_{<t})$ and $V^*_\mu(\yx_{<t}) := V_\mu^{\pi^*_\mu}(\yx_{<t})$.
\end{definition}
Assumption \ref{assumption1} combined with Theorem 6 in \cite{HL11} guarantees the existence of $\pi^*_{\mu}$.
Note that the normalisation term $1 \over \Gamma_t$ does not change the policy, but is used to ensure
that values scale appropriately in the limit. For example, when discounting geometrically we have, $\d_t = \gamma^t$ for some $\gamma \in (0, 1)$
and so $\Gamma_t = {\gamma^t \over {1 - \gamma}}$ and $V^\pi_\mu(\yx_{<t}^{\mu,\pi}) = (1-\gamma)\sum_{k=t}^\infty \gamma^{k-t} r^{\mu,\pi}_k$.

\begin{definition}[Asymptotic Optimality]
Let $\Mc = \left\{\mu_0, \mu_1, \cdots\right\}$ be a finite or countable set of environments and $\v \d$ be a
discount function. A policy $\pi$ is a {\it strong asymptotically optimal} policy in $(\Mc, \v\d)$ if
\eqn{
\label{def_strong}
\lim_{n\to\infty} \left[V^*_\mu(\yx^{\mu,\pi}_{<n}) - V^\pi_{\mu}(\yx^{\mu,\pi}_{<n}) \right] = 0, \text{ for all }
\mu \in \Mc.
}
It is a {\it weak asymptotically optimal} policy if
\eqn{
\label{def_weak}
\lim_{n\to \infty} {1 \over n} \sum_{t=1}^n \left[V^*_{\mu}(\yx^{\mu, \pi}_{<t}) - V^\pi_{\mu}(\yx^{\mu,\pi}_{<t})\right] = 0,
\text{ for all } \mu\in\Mc.
}
\end{definition}
Strong asymptotic optimality demands that the value of a {\it single} policy $\pi$ converges to the value of the optimal policy $\pi^*_\mu$ for {\it all}
$\mu$ in the class. This means that in the limit, a strong asymptotically optimal policy will obtain the maximum value possible in that environments.

Weak asymptotic optimality is similar, but only requires the {\it average} value of the policy $\pi$ to converge to the average value of
the optimal policy. This means that a weak asymptotically optimal policy can still make infinitely many bad mistakes, but must do so for only a fraction
of the time that converges to zero. Strong asymptotic optimality implies weak asymptotic optimality.

While the definition of strong asymptotic optimality is rather natural, the definition of weak asymptotic optimality appears somewhat more arbitrary.
The purpose of the average is to allow the agent to make a vanishing fraction of serious errors over its (infinite) life-time. We believe this is a
necessary condition for an agent to learn the true environment. Of course, it would be possible to insist that the agent make only $o(\log n)$ serious
errors rather than $o(n)$, which would make a stronger version of weak asymptotic optimality. Our choice is the weakest notion of optimality of the above
form that still makes sense, which turns out to be already too strong for some discount rates.

Note that for both versions of optimality an agent would be considered optimal if it actively undertook a policy that led it to an extremely bad ``hell''
state from which it could not escape. Since the state cannot be escaped, its policy would then coincide with the optimal policy and so it would
be considered optimal. Unfortunately, this problem seems to be an unavoidable consequence of learning algorithms in non-ergodic environments
in general, including the currently fashionable PAC algorithms for arbitrary finite Markov decision processes.

\section{Non-Existence of Asymptotically Optimal Policies}

We present the negative theorem in three parts. The first shows that, at least for computable discount functions, there does not
exist a strong asymptotically optimal policy. The second shows that any weak asymptotically optimal policy must be incomputable while
the third shows that there exist discount functions for which even incomputable weak asymptotically optimal policies do not exist.

\begin{theorem}\label{thm_negative}
Let $\Mc$ be the class of all deterministic computable environments and $\v\d$ a computable discount function, then:
\begin{enumerate}
\item There does not exist a strong asymptotically optimal policy in $(\Mc, \v\d)$.
\item There does not exist a computable weak asymptotically optimal policy in $(\Mc, \v\d)$.
\item If $\gamma_k := {1 \over k(k+1)}$ then there does not exist a weak asymptotically optimal policy in $(\Mc, \v\d)$.
\end{enumerate}
\end{theorem}
Part 1 of Theorem \ref{thm_negative} says there is no strong asymptotically optimal policy in the class of all computable
deterministic environments when the discount function is computable. It is likely there exist non-computable discount
functions for which there are strong asymptotically optimal policies. Unfortunately the discount functions for which this
is true are likely to be somewhat pathological and not realistic.

Given that strong asymptotic optimality is too strong, we should search for weak asymptotically optimal policies. Part 2 of Theorem
\ref{thm_negative} shows that any such policy is necessarily incomputable. This result features no real new ideas
and relies on the fact that you can use a computable policy to hand-craft a computable environment in which it does very badly \cite{Leg06}.
In general this approach fails for incomputable policies because the hand-crafted environment will then not be computable.
Note that this does not rule out the existence of a stochastically computable weak asymptotically optimal policy.

It turns out that even weak asymptotic optimality is too strong for some discount functions. Part 3 of Theorem \ref{thm_negative}
gives an example discount function for which no such policy (computable or otherwise) exists. In the next section we introduce a weak asymptotically optimal
policy for geometric (and may be extended to other) discounting.
Note that $\d_k = {1 \over k(k+1)}$ is an example of a discount function where $H_t(p) = \Omega(t)$. It is also analytically easy to work with.

All negative results are proven by contradiction, and follow the same basic form.
\begin{enumerate}
\item Assume $\pi$ is a computable/arbitrary weak/strong asymptotically optimal.
\item Therefore $\pi$ is weak/strong asymptotically optimal in $\mu$ for some particular $\mu$.
\item Construct $\nu$, which is indistinguishable from $\mu$ under $\pi$, but where $\pi$ is not weak/strong asymptotically optimal in $\nu$.
\end{enumerate}
\begin{proof}[\proofof Theorem \ref{thm_negative}, Part 1]
Let $\A = \left\{up, down\right\}$ and $\O = \emptyset$. Now assume some policy $\pi$ is a strong asymptotically optimal policy.
Define an environment $\mu$ by,
\eq{
\mu(\yr_{<t}y_t) = \begin{cases}
{1\over 2} & \text{if } y_t = up \\
0 & \text{if } y_t = down
\end{cases} \in \Rc
}
That is $\mu(\yr_{<t}y_t) \in \Rc$ is the reward given when taking action $y_t$ having previously taken actions $y_{<t}$.
Note that we have omitted the observations as $\O = \emptyset$.
It is easy to see that the optimal policy $\pi^*_\mu(\yr_{<t}) = up$ for all $\yr_{<t}$ with corresponding value $V^*_\mu(\yr_{<t}) = {1 \over 2}$.
Since $\pi$ is strongly asymptotically optimal,
\eqn{
\label{eqn0-1} \lim_{n \to\infty} V_\mu^\pi(\yr^{\mu,\pi}_{1:n}) = {1 \over 2}.
}
Assume there exists a time-sequence $t_1, t_2, t_3,\cdots$ such that $y^{\mu,\pi}_t = down$ (and hence $r^{\mu,\pi}_t = 0$) for all
$t\in \bigcup_{i=1}^\infty [t_i, t_i + H_{t_i}({1 \over 4})]$. Therefore by the definition of the value function,
\eqn{
\label{eqn-no-strong1} V^\pi_\mu(\yr^{\mu, \pi}_{t_i}) &\leq
{1 \over \Gamma_{t_i}} \left[ {1 \over 2} \sum_{k=t_i + H_{t_i}({1 \over 4})+1}^\infty \d_k\right]
= {1 \over 2} \left[1 - {1 \over \Gamma_{t_i}} \sum_{k=t_i}^{t_i + H_{t_i}({1 \over 4})} \d_k\right] \\
\label{eqn-no-strong2} &\leq {1 \over 2} \left[1 - {1 \over 4}\right]
}
where (\ref{eqn-no-strong1}) follows from the definitions of the value function and $\Gamma$, and the assumption in the previous
line. (\ref{eqn-no-strong2}) follows by algebra and the definition of $H_{t_i}({1 \over 4})$.
This contradicts (\ref{eqn0-1}). Therefore for any strong asymptotically optimal policy $\pi$ there exists a $T \in \N$ such that
for all $t \geq T$, $y^{\mu, \pi}_s = up$ for some $s \in [t, t + H_t({1 \over 4})]$. I.e, $\pi$ cannot take sub-optimal action $down$
too frequently. In particular, it cannot take action $down$ for large contiguous blocks of time.
Construct a new environment $\nu$ defined by
\eq{
\nu(\yr_{<t}y_t) = \begin{cases}
\mu(\yr_{<t}y_t) &\text{if } t < T \\
{1 \over 2} & \text{if } y_t = up \\
1 & \text{if } y_t = down \text{ and exist } t' \geq T \text{ such that } t' + H_{t'}({1 \over 4}) \leq t \text{ and } \\
\quad & \qquad y_s = down \text{ } \forall s \in [t', t'+H_{t'}({1 \over 4})] \\
0 & \text{otherwise}
\end{cases}
}
Note that $\nu$ is computable if $H_t({1 \over 4})$ is and that by construction the play-out sequences for $\mu$ and $\nu$ when using policy
$\pi$ are identical. We
now consider the optimal policy in $\nu$. For any $t \geq T$ consider the value of policy $\tilde \pi$ defined by $\tilde \pi(\yr_{<t}) := down$ for
all $\yr_{<t}$.
\eq{
V^{\tilde \pi}_\nu(\yr_{<t}) &= {1 \over \Gamma_t} \left[
\sum_{k=t + H_t({1\over4})}^\infty \d_k \right] \\
&\geq {3 \over 4}.
}
This is because $\tilde \pi$ spends $H_t({1 \over 4}) - 1$ time-steps playing $down$ and receiving reward $0$ before ``unlocking'' a reward of $1$
on all subsequent plays. On the other hand, $V^{\pi}_\nu(\yr^{\nu,\pi}_{<t}) \leq {1 \over 2}$ because $\pi$ can never unlock the reward of $1$ because
it never plays $down$ for a contiguous block of $H_t({1 \over 4})$ time-steps.
By the definition of the optimal policy, $V^{\tilde \pi}_\nu(\yr_{<t}) \leq V^*_\nu(\yr_{<t})$. Therefore
\eq{
V^*_\nu(\yr^{\nu,\pi}_{<t}) - V^\pi_\nu(\yr^{\nu,\pi}_{<t}) \geq {1 \over 4}.
}
Therefore
\eq{
\limsup_{t\to\infty} \left[V^*_\nu(\yr^{\nu,\pi}_{<t}) - V^\pi_\nu(\yr^{\nu,\pi}_{<t})\right] \geq {1 \over 4} \neq 0.
}
Therefore there does not exist an asymptotically optimal policy $\pi$ in $(\Mc, \v\d)$.
\proofend\end{proof}

\begin{proof}[\proofof Theorem \ref{thm_negative}, Part 2]
Let $\A = \left\{up, down\right\}$ and $\O = \emptyset$. Now let $\Mc$ be the class of all computable deterministic
environments and $\v\d$ be an arbitrary discount function. Suppose $\pi$ is computable and consider the environment
$\mu$ defined by
\eq{
\mu(\yr_{<t}y_t) = \begin{cases}
1 & \text{if } y_t \neq \pi(\yr_{>t}) \\
0 & \text{otherwise}
\end{cases}
}
Since $\pi$ is computable $\mu$ is as well. Therefore $\mu \in \Mc$. Now $V^*_\mu(\yr_{<t}) = 1$ for all $\yr_{<t}$ while
$V^\pi_\mu(\yr_{<t}) = 0$. Therefore
$\lim_{n\to\infty} {1 \over n} \sum_{t=1}^n \left|V^*_\mu(\yr_{<t}) - V^\pi_\mu(\yr_{<t})\right| = 1$ and so $\pi$ is
not weakly asymptotically optimal.
\proofend\end{proof}

\begin{proof}[\proofof Theorem \ref{thm_negative}, Part 3]
Recall $\gamma_k = {1 \over k(k+1)}$ and so $\Gamma_t = {1 \over t}$.
Now let $\A = \left\{up, down\right\}$ and $\O = \emptyset$. Define $\mu$ by
\eq{
\mu(\yr_{<t}y_t) = \begin{cases}
{1 \over 2} & \text{if } y_t = up \\
{1 \over 2} - \epsilon & \text{if } y_t = down
\end{cases}
}
where $\epsilon \in (0, {1 \over 2})$ will be chosen later. As before, $V_\mu^*(\yr_{<t}) = {1 \over 2}$. Assume $\pi$ is weakly asymptotically optimal.
Therefore
\eqn{
\label{eqn-thm-no-weak-1} \lim_{n\to\infty} {1 \over n} \sum_{t=1}^n V^\pi_\mu(\yr_{<t}^{\mu,\pi}) = {1 \over 2}.
}
We show by contradiction that $\pi$ cannot explore (take action $down$) too often.
Assume there exists an infinite time-sequence $t_1, t_2, t_3,\cdots$ such that $\pi(\yr_{<t}^{\mu,\pi}) = down$ for all
$t \in \bigcup_{i=1}^\infty [t_i, 2t_i]$. Then
for $t \in [t_i, {3 \over 2} t_i]$ we have
\eqn{
V^\pi_\mu(\yr_{<t}^{\mu,\pi}) &\equiv {1 \over \Gamma_t} \sum_{k=t}^\infty \d_k r^{\mu,\pi}_k
\label{eqn-thm-no-weak-3} \leq  t \left[({1 \over 2} - \epsilon) \sum_{k=t}^{2t_i} \d_k +
{1 \over 2} \sum_{k=2t_i+1}^\infty \d_k\right] \\
\label{eqn-thm-no-weak-4} &= {1 \over 2} - \epsilon\left[1 - { t \over 2t_i +1}\right]
< {1 \over 2} - {\epsilon \over 4}
}
where (\ref{eqn-thm-no-weak-3}) is the definition of the value function and the previous assumption and definition of $\mu$. (\ref{eqn-thm-no-weak-4}) by algebra and since $t \in [t_i, {3 \over 2} t_i]$. Therefore
\eqn{
\label{eqn-thm-no-weak-6} {1 \over {2t_i}} \sum_{t=1}^{2t_i} V^\pi_\mu(\yr^{\mu,\pi}_{<t}) < {1 \over 2t_i}
\left[\sum_{t=1}^{t_i-1} {1 \over 2} + \sum_{t=t_i}^{{3 \over 2} t_i - 1} \left({1 \over 2} - {\epsilon\over 4}\right) +
\sum_{t={3 \over 2} t_i}^{2t_i} {1 \over 2} \right] = {1 \over 2} - {1 \over 16} \epsilon.
}
The first inequality follows from (\ref{eqn-thm-no-weak-4}) and because the maximum value of any
play-out sequence in $\mu$ is ${1 \over 2}$.
The second by algebra. Therefore
$\liminf_{n\to\infty} {1\over n} \sum_{t=1}^n V^\pi_\mu(\yr^{\mu,\pi}_{<t}) <
 {1 \over 2} - {1 \over 16} \epsilon  < {1 \over 2}$, which contradicts (\ref{eqn-thm-no-weak-1}).
Therefore there does not exist a time-sequence $t_1 < t_2 < t_3 < \cdots$
such that $\pi(\yr_{<t}^{\mu,\pi}) = down$ for all $t \in \bigcup_{i=1}^\infty [t_i, 2t_i]$.

So far we have shown that $\pi$ cannot ``explore'' for $t$ consecutive time-steps starting at time-step $t$, infinitely often. We now construct
an environment similar to $\mu$ where this is required. Choose $T$ to be larger than the
last time-step $t$ at which
$y^{\mu,\pi}_s = down$ for all $s \in [t, 2t]$
Define $\nu$ by
\eq{
\nu(\yr_{<t}y_t) = \begin{cases}
\mu(\yr_{<t}y_t) & \text{if } t < T \\
{1 \over 2} & \text{if } y_t = down \text{ and there does not exist } t' \geq T \\
\quad & \qquad \text{ such that } y_s = down\forall s \in [t',2t'] \\
1 & \text{if } y_t = down \text{ and exists } t' \geq T \text{ such that } 2t' < t \text{ and } \\
\quad & \qquad y_s = down \forall s \in [t', 2t'] \\
{1 \over 2} - \epsilon & \text{otherwise}
\end{cases}
}
Now we compare the values in environment $\nu$ of $\pi$ and $\pi_\nu^*$ at times $t \geq T$. Since $\pi$ does
not take action $down$ for $t$ consecutive time-steps at any time after $T$, it never ``unlocks'' the reward of 1 and so
$V^\pi_\nu(\yr^{\nu,\pi}_{<t}) \leq {1 \over 2}$.
Now let $\tilde \pi(\yr_{<t}) = down$ for all $\yr_{<t}$. Therefore, for $t \geq 2T$,
\eqn{
V^{\tilde \pi}_\nu(\yr^{\nu,\pi}_{<t}) &\equiv {1 \over \Gamma_t} \sum_{k=t}^\infty \d_k r^{\nu,\tilde \pi}_k
\label{eqn-thm-no-weak-8} \geq t \left[\left({1 \over 2} - \epsilon\right) \sum_{k=t}^{2t-1} \d_k + \sum_{k=2t}^\infty \d_k \right] \\
\label{eqn-thm-no-weak-9} &= t \left[\left({1 \over 2} -\epsilon\right) \left({1 \over t} - {1 \over 2t}\right) + {1 \over 2t} \right]
= {3 \over 4} - {1 \over 2} \epsilon
}
where (\ref{eqn-thm-no-weak-8}) follows by the definition of $\nu$ and $\tilde \pi$. (\ref{eqn-thm-no-weak-9}) by the definition of $\d_k$ and algebra.
Finally, setting $\epsilon = {1 \over 4}$ gives
$V^{\tilde \pi}_\nu(\yr^{\nu,\pi}_{<t}) \geq {5 \over 8} = {1 \over 2} + {1 \over 8}$.
Since $V^*_\nu \geq V^{\tilde \pi}_\nu$, we get
$
V^*_\nu(\yr^{\nu,\pi}_{<t}) - V^\pi_\nu(\yr^{\nu,\pi}_{<t}) \geq V^{\tilde \pi}_\nu(\yr^{\nu,\pi}_{<t}) - V^\pi_\nu(\yr^{\nu,\pi}_{<t}) \geq {1 \over 8}$.
Therefore
$
\limsup_{n\to\infty} {1 \over n} \sum_{t=1}^n \left[V^*(\yr^{\nu,\pi}_{<t}) - V^\pi_\nu(\yr^{\nu,\pi}_{<t})\right] \geq {1 \over 8}$, and so
$\pi$ is not weakly asymptotically optimal.
\proofend\end{proof}
We believe it should be possible to generalise the above to computable discount functions with $H_t(p) > c_pt$ with $c_p > 0$ for infinitely many $t$,
but the proof will likely be messy.

\section{Existence of Weak Asymptotically Optimal Policies}

In the previous section we showed there did not exist a strong asymptotically optimal policy (for most discount functions)
and that any weak asymptotically optimal policy must be incomputable. In this section we show that a weak
asymptotically optimal policy exists for geometric discounting (and is, of course, incomputable).

The policy is reminiscant of $\epsilon$-exploration in finite state MDPs (or UCB for bandits) in that it spends
most of its time exploiting the information it already knows, while
still exploring sufficiently often (and for sufficiently long) to detect any significant errors in its model.

The idea will be to use a model-based policy that chooses its
current model to be the first environment in the model class (all computable deterministic environments) consistent
with the history seen so far. With increasing probability it takes the best action according to this policy,
while still occasionally exploring randomly. When it explores it always does so in bursts of increasing length.
\begin{definition}[History Consistent]
A deterministic environment $\mu$ is {\it consistent} with history $\yx_{<t}$ if
$\mu(\yx_{<k}y_k) = x_k, \text{ for all } k < t$.
\end{definition}

\begin{definition}[Weak Asymptotically Optimal Policy]  \label{defn_det_asym}
Let $\A = \left\{0, 1\right\}$ and $\Mc = \left\{\mu_1, \mu_2, \mu_3, \cdots \right\}$ be a countable class of deterministic environments.
Define a probability measure $P$ on $\B^\infty$ inductively by,
$P(z_n = 1 | z_{<n}) := {1 \over n}, \text{ for all } z_{<n} \in \B^{n-1}$.
Now let $\chi \in \B^\infty$ be sampled from $P$ and define $\bar \chi, \dot\chi^h \in \B^\infty$ by
\eqn{
\nonumber \bar \chi_k &:= \begin{cases}
1 & \text{if } k \in \bigcup_{i : \chi_i = 1} [i, i + \log i] \\
0 & \text{otherwise}
\end{cases} &
\label{chi_bar} \dot\chi^h_k &:= \begin{cases}
0 & \text{if } \bar\chi_{k:k+h} = 0^{h+1} \\
1 & \text{otherwise}
\end{cases}
}
Next let $\psi$ be sampled from the uniform measure (each bit of $\psi$ is independently sampled from a Bernoulli $1/2$ distribution)
and define a policy $\pi$ by,
\eqn{
\pi(\yx_{<t}) := \begin{cases}
\pi^*_{\nu_t}(\yx^{\pi,\mu}_{<t}) & \text{if } \bar \chi_t = 0 \\
\psi_t & \text{otherwise}
\end{cases}
}
where $\nu_t = \mu_{i_t}$ with $i_t = \min\left\{i: \mu_i \text{ consistent with history } \yx^{\pi,\mu}_{<t} \right\} < \infty$. Note that $i_t$
is always finite because there exists an $i$ such that $\mu_i = \mu$, in which case $\mu_i$ is necessarily consistent with $\yx^{\pi,\mu}_{<t}$.
\end{definition}
Intuitively, $\chi_k = 1$ at time-steps when the agent will explore for $\log k$ time-steps.
$\bar\chi_k = 1$ if the agent is exploring at time $k$ and $\psi_k$ is the action taken if exploring at time-step $k$. $\dot\chi$ will be
used later, with $\dot \chi^h_k = 1$ if the agent will explore at least once in the interval $[k, k+h]$.
If the agent is not exploring then it acts according to the optimal policy for the first consistent environment in $\Mc$.

\begin{theorem}\label{thm_weak_deterministic_optimal}
Let $\d_t = \gamma^t$ with $\gamma \in (0, 1)$ (geometric discounting) then
the policy defined in Definition \ref{defn_det_asym} is weakly asymptotically optimal in the class
of all deterministic computable environments with probability 1.
\end{theorem}
Some remarks:
\begin{enumerate}
\item That $\A = \left\{0, 1\right\}$ is only convenience, rather than necessity. The policy is easily generalised to arbitrary finite $\A$.
\item $\pi$ is essentially a stochastic policy. With some technical difficulties it is possible to construct an
equivalent deterministic policy. This is done by choosing $\chi$ to be any $P$-Martin-L\"of random sequence and $\psi$ to be a sequence that is
Martin-L\"of random w.r.t to the uniform measure. The theorem then
holds for {\it all} deterministic environments. The proof is somewhat delicate and may not extend nicely to stochastic environments.
For an introduction to Kolmogorov complexity and Martin-L\"of randomness, see \cite{LV08}. For a reason why the stochastic case may not
go through as easily, see \cite{HM07}.
\item The policy defined in Definition \ref{defn_det_asym} is not computable for two reasons. First, because it relies on the stochastic
sequences $\chi$ and $\psi$. Second, because the operation of finding the first environment consistent with the history is not
computable.\footnote{The class of computable environments is not recursively enumerable \cite{LV08}.} We do not know if there exists a weak asymptotically optimal policy
that is computable when given access to a random number generator (or if it is given $\chi$ and $\psi$).
\item The bursts of exploration are required for optimality. Without them it will be possible to construct counter-example environments similar
to those used in part 3 of Theorem \ref{thm_negative}.
\end{enumerate}
Before the proof we require some more definitions and lemmas. Easier proofs are omitted.
\begin{definition}[$h$-Difference]
Let $\mu$ and $\nu$ be two environments consistent with history $\yx_{<t}$, then $\mu$ is {\it $h$-different} to $\nu$ if
there exists $\yx_{t:t+h}$ satisfying
\eq{
y_k &= \pi^*_\mu(\yx_{<k}) \text{ for all } k \in [t, t+h], \\
x_k &= \mu(\yx_{<k}y_k) \text{ for all } k \in [t, t+h], \\
x_k &\neq \nu(\yx_{<k}y_k) \text{ for some } k \in [t, t+h].
}
\end{definition}
Intuitively, $\mu$ is $h$-different to $\nu$ at history $\yx_{<t}$ if playing the optimal policy for $\mu$ for $h$ time-steps
makes $\nu$ inconsistent with the new history. Note that $h$-difference is {\it not} symmetric.

\begin{lemma}\label{lem_tech2}
If $a_n \in [0,1]$ and $\limsup_{n\to\infty} {1 \over n} \sum_{i=1}^n a_n = \epsilon$ and $\alpha \in \B^\infty$ is
an indicator sequence with $\alpha_i := \ind{a_i \geq \epsilon /4}$,\footnote{$\ind{expression} = 1$ if $expression$ is true and $0$ otherwise.} then
$\prod_{i=1}^\infty \left[1 - {\alpha_i \over i} \right] = 0$.
\end{lemma}
See the appendix for the proof.
\begin{lemma}\label{lem_chi}
Let $a_1, a_2, a_3,\cdots$ be a sequence with $a_n \in [0, 1]$ for all $n$.
The following properties of $\chi$ are true with probability $1$.
\begin{enumerate}
\item For any $h$, $\limsup_{n\to \infty} {1 \over n} \#1(\dot\chi^h_{1:n}) = 0$.
\item If $\limsup {1 \over n} \sum_{i=1}^n a_i = \epsilon > 0$ and $\alpha_i := \ind{a_i > \epsilon/2}$ then $\alpha_i = \chi_i = 1$ for infinitely many $i$.
\end{enumerate}
\end{lemma}
\begin{proof}
1. Let $i \in \N$, $\epsilon > 0$ and $E_i^\epsilon$ be the event that $\#1(\dot\chi^h_{1:2^i}) > 2^i \epsilon$.
Using the definition of $\dot\chi^h$ to compute the expectation $\E\left[\#1(\dot\chi^h_{1:2^i})\right] < i(i + 1)h$ and applying the Markov inequality
gives that $P(E_i^\epsilon) < i(i+1) h 2^{-i} / \epsilon$. Therefore $\sum_{i \in \N} P(E_i^\epsilon) < \infty$.
Therefore the Borel-Cantelli lemma gives that $E_i^\epsilon$ occurs for only finitely many $i$ with probability $1$.
We now assume that $\limsup_{n\to\infty} {1 \over n} \#1(\dot\chi^h_{1:n}) > 2 \epsilon > 0$ and show that $E_i^\epsilon$ must occur infinitely often.
By the definition of $\limsup$ and our assumption we have that there exists a sequence
$n_1, n_2, \cdots$ such that $\#1(\dot\chi^h_{1:n_i}) > 2 n_i \epsilon$ for all $i \in \N$.
Let $n^+ := \min_{k\in\N} \left\{2^k : 2^k \geq n\right\}$ and note that $\#1(\dot\chi^h_{1:n_i^+}) > n_i^+ \epsilon$, which is
exactly $E^\epsilon_{\log n_i^+}$. Therefore there exist infinitely many $i$ such that $E_i^\epsilon$ occurs and so
$\limsup_{n\to\infty} {1 \over n} \#1(\dot\chi^h_{1:n}) = 0$ with probability 1.\\
2. The probability that $\alpha_i = 1 \implies \chi_i = 0$ for all $i \geq T$ is
$P(\alpha_i = 1 \implies \chi_i = 0 \forall i \geq T) = \prod_{i=T}^\infty \left(1 - {\alpha_i \over i}\right) =: p = 0$,
by Lemma \ref{lem_tech2}.
Therefore the probability that $\alpha_i = \chi_i = 1$ for only
finitely many $i$ is zero. Therefore there exists infinitely many $i$ with $\alpha_i = \chi_i = 1$ with probability $1$, as required.
\proofend\end{proof}
\begin{lemma}[Approximation Lemma]\label{lem_approximation}
Let $\pi_1$ and $\pi_2$ be policies, $\mu$ an environment and $h \geq H_t(1-\epsilon)$. Let $\yx_{<t}$ be an arbitrary history and
$\yx^{\mu, \pi_i}_{t:t + h}$ be the future action/observation/reward triples when playing policy $\pi_i$.
If $\yx^{\pi_1, \mu}_{t:t+h} = \yx^{\pi_2, \mu}_{t:t+h}$ then
$|V^{\pi_1}_\mu(\yx_{<t}) - V^{\pi_2}_\mu(\yx_{<t})| < \epsilon$.
\end{lemma}

\begin{proof}
By the definition of the value function,
\eqn{
\label{lem_ap1}|V^{\pi_1}_\mu(\yx_{<t}) - V^{\pi_2}_\mu(\yx_{<t})| &\leq
{1 \over \Gamma_t} \sum_{i=t}^\infty \d_i \left| r^{\pi_1, \mu}_i - r^{\pi_2, \mu}_i  \right| \\
\label{lem_ap2}&= {1 \over \Gamma_t} \sum_{i=t+h+1}^\infty \d_i \left|r^{\pi_1, \mu}_i - r^{\pi_2, \mu}_i \right|
\leq {1 \over \Gamma_t} \sum_{i=t+h+1}^\infty \d_i < \epsilon
}
(\ref{lem_ap1}) follows from the definition of the value function. (\ref{lem_ap2}) since
$r^{\pi_1, \mu}_i = r^{\pi_2, \mu}_i$ for $i \in [t, t+h]$, rewards are bounded in $[0, 1]$ and by the definition of $h := H_t(1-\epsilon)$ (Definition \ref{defn_discount}).
\proofend\end{proof}

Recall that $\pi^*_\mu$ and $\pi^*_\nu$ are the optimal policies in environments $\mu$ and $\nu$
respectively (see Definition \ref{defn_optimal}).
\begin{lemma}[$h$-difference]\label{lem_h_different}
If $|V^{\pi^*_{\mu}}_\mu(\yx^{\pi, \mu}_{<t}) - V^{\pi^*_\nu}_\mu(\yx^{\pi, \mu}_{<t})| > \epsilon$
then $\mu$ is $H_t(1 - \epsilon)$-different to $\nu$ on $\yx^{\pi,\mu}$.
\end{lemma}

\begin{proof}
Follows from the approximation lemma.
\proofend\end{proof}

We are now ready to prove the main theorem.

\begin{proof}[\proofof Theorem \ref{thm_weak_deterministic_optimal}]
Let $\pi$ be the policy defined in Definition \ref{defn_det_asym} and
$\mu$ be the true (unknown) environment. Recall that $\nu_t = \mu_{i_t}$ with
$i_t = \min\left\{i: \mu_i \text{ consistent with history } \yx^{\pi,\mu}_{<t} \right\}$ is
the first model consistent with the history $\yx^{\pi,\mu}_{<t}$ at time $t$ and is used by $\pi$ when not exploring.
First we claim there exists a $T \in \N$ and environment $\nu$ such that $\nu_t = \nu$ for all $t \geq T$. Two facts,
\begin{enumerate}
\item If $\mu_i$ is inconsistent with history $\yx^{\pi,\mu}_{<t}$ then it is also inconsistent with $\yx^{\pi,\mu}_{<t+h}$ for
all $h \in \N$.
\item $\mu$ is consistent with $\yx^{\pi,\mu}_{<t}$ for all $\pi, t$.
\end{enumerate}
By 1) we have that the sequence $i_1, i_2, i_3, \cdots$ is monotone increasing. By 2) we have that the sequence is
bounded by $i$ with $\mu_i = \mu$. The claim follows since any bounded monotone sequence of natural numbers converges in finite time.
Let $\nu := \nu_\infty$ be the environment to which $\nu_1, \nu_2, \nu_3, \cdots$ converges to. Note that $\nu$ must be consistent
with history $\yx^{\mu,\pi}_{<t}$ for all $t$.
We now show by contradiction that the optimal policy for $\nu$ is weakly asymptotically
optimal in environment $\mu$. Suppose it were not, then
\eqn{
\label{eqn9-1} \limsup_{n\to\infty} {1\over n} \sum_{t=1}^n \left[V^*_\mu(\yx^{\pi,\mu}_{<t}) - V^{\pi^*_\nu}_\mu(\yx^{\pi,\mu}_{<t})\right] = \epsilon > 0.
}
Let $\alpha \in \B^\infty$ be defined by $\alpha_t := 1$ if and only if,
\eqn{
\label{eqn10-1} \left[V^*_\mu(\yx^{\pi,\mu}_{<t}) - V^\pi_\mu(\yx^{\pi,\mu}_{<t})\right] \geq \epsilon / 4.
}
By Lemma \ref{lem_chi} there exists (with probability one) an infinite sequence $t_1, t_2, t_3, \cdots$ for which $\chi_k = \alpha_k = 1$.
Intuitively we should view time-step $t_k$ as the start of an ``exploration'' phase where the agent explores
for $\log t_k$ time-steps.
Let $h := H_{t_k}(1 - \epsilon / 4) = \left\lceil \log(\epsilon/4)/\log\gamma \right\rceil$, which importantly is independent of $t_k$ (for geometric discounting).
Since $\log t_k \to \infty$ we will assume that $\log t_k \geq h$ for all $t_k$.
Therefore $\bar \chi_i = 1$ for all $i \in \bigcup_{k=1}^\infty [t_k, t_k + h]$.
Therefore by the definition of $\pi$, $\pi(\yx^{\pi,\mu}_{<i}) = \psi_i$ for $i \in \bigcup_{k=1}^\infty [t_k, t_k + h]$.
By Lemma \ref{lem_h_different} and Equation (\ref{eqn10-1}), $\mu$ is $h$-different to $\nu$ on history $\yx^{\pi,\mu}_{<t_k}$.
This means that if there exists a $k$ such that $\pi$ plays according to the optimal policy for $\mu$ on all time-steps
$t \in [t_k, t_k + h]$ then $\nu$ will be inconsistent with the history $\yx^{\mu,\pi}_{1:t_k + h}$, which is a contradiction.
We now show that $\pi$ does indeed play according to the optimal policy for $\mu$ for all time-steps $t \in [t_k, t_k + h]$ for at least one $k$.
Formally, we show the following holds with probability $1$ for some $k$.
\eqn{
\label{eqn11-1} \psi_{i} \equiv \pi(\yx^{\pi,\mu}_{<i}) = \pi^*_\mu(\yx^{\pi,\mu}_{<i}), \text{ for all } i \in [t_k, t_k + h].
}
Recall that $\psi \in \B^\infty$ where $\psi_i \in \B$ is identically independently distributed according to a Bernoulli($1\over 2$) distribution.
Therefore $P(\psi_i = \pi^*_\mu(\yx^{\pi,\mu}_{<i})) = {1 \over 2}$. Therefore
$p:=P(\psi_{i} = \pi^*_\mu(\yx^{\pi,\mu}_{<i}) \forall i \in [t_k, t_k + h])
= \prod_{i=t_k}^{t_k + h} P(\psi_{i} = \pi^*_\mu(\yx^{\pi,\mu}_{<i})) =
2^{-h-1} > 0$
and
$
P(\forall k \exists i \in [t_k, t_k + h] \text{ with }
\psi_{i} \neq \pi^*_\mu(\yx^{\pi,\mu}_{<i})) = \prod_{k=1}^\infty (1-p)= 0
$.
Therefore Equation (\ref{eqn11-1}) is satisfied for some $k$ with probability 1 and so Equation (\ref{eqn9-1}) leads to a contradiction. Therefore
\eqn{
\label{eqn_main0} \lim_{n\to\infty} {1\over n} \sum_{t=1}^n \left[V^*_\mu(\yx^{\pi,\mu}_{<t}) - V^{\pi_\nu^*}_\mu(\yx^{\pi,\mu}_{<t})\right] = 0.
}
We have shown that the optimal policy for $\nu$ has similar $\mu$-values to the optimal policy for $\mu$. We now show that $\pi$ acts according
to $\pi^*_\nu$ sufficiently often that it too has values close to those of the optimum policy for the true environment, $\mu$.
Let $\epsilon > 0$, $h := H_t(1-\epsilon)$ and $t \geq T$. If $\dot\chi^h_t = 0$ then by the definition of $\pi$ and the approximation lemma we obtain
\eqn{
\label{eqn_main1} \left|V_\mu^{\pi_\nu^*}(\yx^{\pi,\mu}_{<t}) - V_\mu^{\pi}(\yx^{\pi,\mu}_{<t})\right| < \epsilon.
}
Therefore
\eqn{
\label{eqn_main2} \limsup_{n\to\infty} {1\over n} \sum_{t=1}^n \left|V^{\pi^*_\nu}_\mu(\yx^{\pi,\mu}_{<t}) - V^{\pi}_\mu(\yx^{\pi,\mu}_{<t})\right| &\leq
\limsup_{n\to\infty} {1\over n} \left|\sum_{t=1}^{T-1} 1 + \sum_{t=T}^n \left[\dot\chi^h_t(1 - \epsilon) + \epsilon \right]\right| \\
\label{eqn_main3}&= \epsilon + (1 - \epsilon) \limsup_{n\to\infty} {1 \over n} \#1(\dot\chi^h_{T:n}) \\
\label{eqn_main4}&= \epsilon
}
where (\ref{eqn_main2}) follows since values are bounded in $[0, 1]$ and Equation (\ref{eqn_main1}). (\ref{eqn_main3}) follows by algebra.
(\ref{eqn_main4}) by part 1 of Lemma \ref{lem_chi}. By sending $\epsilon \to 0$,
\eqn{
\label{eqn_main5} \lim_{n\to\infty} {1\over n} \sum_{t=1}^n \left[V^{\pi^*_\nu}_\mu(\yx^{\pi,\mu}_{<t}) - V^{\pi}_\mu(\yx^{\pi,\mu}_{<t})\right] = 0.
}
Finally, combining Equations (\ref{eqn_main0}) and (\ref{eqn_main5}) gives the result.
\proofend\end{proof}
We expect this theorem to generalise without great difficulty to  discount functions satisfying $H_t(p) < c_p \log(t)$ for all $p$. There will be two key
changes. First, extend the exploration time to some function $E(t)$ with $E(t) \in O(H_p(t))$ for all $p$. Second, modify the probability of exploration
to ensure that Lemma \ref{lem_chi} remains true.

\section{Discussion}

\subsubsect{Summary}
Part 1 of Theorem \ref{thm_negative} shows that no policy can be strongly asymptotically optimal for any computable discount function.
The key insight is that strong asymptotic optimality essentially implies exploration must eventually cease. Once this occurs, the
environment can change without the agent discovering the difference and the policy will no longer be optimal.

A weaker notion of asymptotic optimality, that a policy be optimal on average in the limit, turns out to be more interesting. Part 2 of Theorem
\ref{thm_negative} shows that no weak asymptotically optimal policy can be computable. We should not be surprised by this
result. Any computable policy can be used to construct a computable environment in which that policy does very badly. Note that by
computable here we mean deterministic and computable. There may be computable stochastic policies that are weakly asymptotically optimal, but
we feel this is unlikely.

Part 3 of Theorem \ref{thm_negative}, shows that even weak asymptotically optimal policies need not exist if the discount
function is sufficiently far-sighted.
On the other hand, Theorem \ref{thm_weak_deterministic_optimal} shows that weak asymptotically optimal policies do exist for some discount rates,
in particular, for the default geometric discounting.
These non-trivial and slightly surprising result shows that choice of discount function is crucial to the existence of
weak asymptotically optimal policies.
Where weak asymptotically optimal policies do exist, they must explore infinitely often and in increasing contiguous bursts of exploration where
the length of each burst is dependent on the discount function.

\subsubsect{Consequences}
It would appear that Theorem \ref{thm_negative} is problematic for artificial general intelligence. We cannot construct
incomputable policies, and so we cannot construct weak asymptotically optimal policies. However, this is not as problematic as it may seem. There are
a number of reasonable counter arguments:
\begin{enumerate}
\item We may be able to make stochastically computable policies that are asymptotically optimal. If the existence of true random noise is assumed
then this would be a good solution.
\item The counter-example environment constructed in part 2 of Theorem \ref{thm_negative} is a single environment roughly as complex as the policy
itself. Certainly, if the world were adversarial this would be a problem, but in general this appears not to be the case.
On the other hand, if the environment is a learning agent itself, this could result in a complexity arms race without bound.
There may exist a computable weak asymptotically optimal policy in some extremely large class of environments. For example,
the algorithm of Section 4 is stochastically computable when the class of environments is recursively enumerable and contains only computable
environments. A natural (and already quite large) class satisfying these properties is finite-state Markov decision processes with
$\left\{0, 1\right\}$-valued transition functions and rational-valued rewards.
\item While it is mathematically pleasant to use asymptotic behaviour to characterise optimal general intelligent behaviour, in practise we usually care
about more immediate behaviour. We expect that results, and even (parameter free) formal definitions of intelligence satisfying this need will
be challenging, but worthwhile.
\item Accept that even weak asymptotic optimality is too strong and find something weaker, but still useful.
\end{enumerate}

\subsubsect{Relation to AIXI}
The policy defined in Section 4 is not equivalent to AIXI \cite{Hut04}, which is also incomputable. However, if the computable environments in $\Mc$ are ordered by complexity
then it is likely the two will be quite similar. The key difference is the policy defined in this paper will continue to explore whereas
it was shown in \cite{Ors10} that AIXI eventually ceases exploration in some environments and some histories. We believe, and a
proof should not be too hard, that AIXI will become weakly asymptotically optimal if an exploration component is added similarly as in Section 4.

We now briefly compare the self-optimising property in \cite{Hut02} to strong asymptotic optimality.
A policy $\pi$ is self-optimising in a class $\Mc$ if
$\lim_{t\to\infty} \left[V^*_\mu(\yx_{<t}) - V_\mu^\pi(\yx_{<t})\right] = 0$ for any infinite history $\yx_{1:\infty}$
and for all $\mu \in \Mc$.
This is similar to strong asymptotic optimality, but convergence must be on all histories, rather than
the histories actually generated by $\pi$. This makes the self-optimising property a substantially stronger form of
optimality than strong asymptotic optimality. It has been proven that if there exists self-optimising policy for a
particular class, then AIXI is also self-optimising in that class \cite{Hut02}.

It is possible to define a weak version of the self-optimising property by insisting that
$\lim_{n\to\infty} {1 \over n} \sum_{t=1}^n \left[V^*_\mu(\yx_{<t}) - V_\mu^\pi(\yx_{<t})\right] = 0$ for all $\yx_{1:\infty}$
and all $\mu \in \Mc$. It can then be proven that the existence of a weak self-optimising policy would imply that AIXI
were also weakly self-optimising. However, the policy defined in Section 4 cannot be modified to have the weak
self-optimising property. It must be allowed to choose its actions itself. This is consistent with the work
in \cite{Ors10} which shows that AIXI cannot be weakly asymptotically optimal, and so cannot be weak self-optimising
either.

\subsubsect{Discounting}
Throughout this paper we have assumed rewards to be discounted according to a summable discount function. A very natural alternative to discounting,
suggested in \cite{HL07}, is to restrict interest to environments satisfying $\sum_{k=1}^\infty r^{\mu,\pi}_k \leq 1$. Now the goal of the
agent is simply to maximise summed rewards. In this setting it is easy to see that the positive theorem is lost while all negative ones still hold!
This is unfortunate, as discounting presents a major philosophical challenge. How to choose a discount function?

\subsubsect{Assumptions/Limitations}
Assumption \ref{assumption1} ensures that $\A$ and $\O$ are finite. All negative results go through for countable $\A$ and $\O$. The optimal
policy of Section 4 may not generalise to countable $\A$.
We have also assumed bounded reward and discrete time. The first seems reasonable while the second allows for substantially easier analysis.
Additionally we have only considered deterministic computable environments. The stochastic case is unquestionably interesting.
We invoked Church thesis to assert that computable stochastic environments are essentially the largest class of interesting environments.

Many of our Theorems are only applicable to computable discount functions. All well-known discount function in use today are computable. However
\cite{Hut04} has suggested $\gamma_t = 2^{-K(t)}$, where $K(t)$ is the (incomputable) prefix Kolmogorov complexity of $t$, may have nice theoretical
properties.

\subsubsect{Open Questions}
A number of open questions have arisen during this research. In particular,
\begin{enumerate}
\item Prove Theorem \ref{thm_negative} for a larger class of discount functions.
\item Prove or disprove the existence of a weak asymptotically optimal stochastically computable policy for some discount function in the
class of deterministic computable environments.
\item Modify the policy of Section 4 to the larger class of stochastically computable environments. We believe this to be
possible along the lines of \cite{HR08}, but inevitably the analysis will be messy and complex.
\item Extend Part 3 of Theorem \ref{thm_negative} and Theorem \ref{thm_weak_deterministic_optimal} to a complete classification of
discount functions according
to whether or not they admit a weak asymptotically optimal policy in the class of computable environments.
\item Prove that AIXI is weakly asymptotically optimal when augmented with an exploration component as in Section 4.
\item Define and study other formal measures of optimality/intelligence.
\end{enumerate}

\subsubsect{Acknowledgements}
We thank Laurent Orseau, Wen Shao and reviewers for valuable
feedback on earlier drafts and the Australian Research Council
for support under grant DP0988049.


\begin{small}

\end{small}

\appendix
\section{Technical Proofs}

\begin{lemma}\label{lem_tech1}
Let $A = \left\{a_1, a_2, \cdots, a_n\right\}$ with $a \in [0, 1]$ for all $a \in A$. If ${1\over n} \sum_{a \in A} a \geq \epsilon$ then
\eq{
\left|\left\{a \in A : a \geq {\epsilon \over 2} \right\}\right| > n {\epsilon \over 2}
}

\end{lemma}

\begin{proof}
Let $A_{>} := \left\{a \in A : a \geq {\epsilon \over 2}\right\}$ and
$A_{<} := A - A_{>}$. Therefore
\eq{
n\epsilon &\leq \sum_{a\in A} a = \sum_{a\in A_{<}} a + \sum_{a\in A_{>}} a \\
& \leq \sum_{a\in A_{<}} {\epsilon \over 2} + \sum_{a \in A_{>}} 1 \\
& = |A_{<}| {\epsilon \over 2} + |A_{>}|
}
By rearranging and algebra, $\left|\left\{ a \in A : a \geq {\epsilon \over 2} \right\} \right| \equiv |A_{>}| > n{\epsilon \over 2}$ as required.
\proofend\end{proof}

\begin{proof}[\proofof Lemma \ref{lem_tech2}] First,
\eqn{
\prod_{i=1}^\infty\left[1-{\alpha_i \over i} \right] & \leq \exp\left[-\sum_{i=1}^\infty {\alpha_i \over i} \right]
\label{eqn13-1}
}
Equation (\ref{eqn13-1}) follows since $1 - a \leq \exp(-a)$ for all $a$.

Now since $\limsup_{n\to\infty} {1 \over n} \sum_{i=1}^n a_n = \epsilon$, we have for any $N$ there exists an $n > N$
such that ${1 \over n} \sum_{i=1}^n a_n > {\epsilon \over 2}$.
Let $n_1 = 0$ then inductively choose
$n_i = \min\left\{n: n > {8(n_{i-1}+1) \over \epsilon} \wedge {1 \over n} \sum_{i=1}^n a_i > {\epsilon \over 2} \right\}$

By Lemma \ref{lem_tech1},
\eqn{
\label{eqn-tech2-1} \left|\left\{i \leq n_j : a_i \geq {\epsilon \over 4}\right\}\right| \geq n_j{\epsilon \over 4}
}
Therefore
\eqn{
\label{eqn-tech2-3} \sum_{i=n_j+1}^{n_{j+1}} {\alpha_i \over i} &\geq \sum_{i=n_{j+1} - {\epsilon \over 4} n_{j+1} + n_j + 1}^{n_{j+1}} {1 \over n_{j+1}} \\
\label{eqn-tech2-4} &\geq \sum_{i=(1 - {\epsilon \over 8})n_{j+1}}^{n_{j+1}} {1 \over n_{j+1}} = {\epsilon \over 8}
}
Equation (\ref{eqn-tech2-3}) follows from (\ref{eqn-tech2-1}) and because $1 \over i$ is a decreasing function. (\ref{eqn-tech2-4}) follows from
the definition of $n_j$ and algebra. Therefore
\eqn{
\nonumber \sum_{i=1}^\infty {\alpha_i \over i} &= \lim_{k\to\infty} \sum_{j=1}^k \sum_{i=n_j+1}^{n_{j+1}} {\alpha_i \over i} \\
\label{eqn-tech2-6} &\geq \lim_{k\to\infty} \sum_{j=1}^k {\epsilon \over 8} = \infty
}
Finally, substituting Equation (\ref{eqn-tech2-6}) into (\ref{eqn13-1}) gives,
\eq{
\prod_{i=1}^\infty \left[1 - {\alpha_i \over i}\right] = 0
}
as required.
\proofend\end{proof}

\section{Table of Notation}

\begin{tabular}{|p{1.6cm} | p{12.5cm}|}
\hline
{\bf Symbol} & {\bf Description} \\
$\A$ & Set of possible actions \\
$\O$ & Set of possible observations \\
$\Rc$ & Set of possible rewards \\
$\mu, \nu$ & Environments \\
$y$ & An action. \\
$x$ & An observation/reward pair \\
$r$ & A reward \\
$o$ & An observation \\
$\ind{expr}$ & The delta function. $\ind{expression} = 1$ if $expression$ is true and $0$ otherwise. \\
$\neg b$ & The {\it not} function. $\neg 0 = 1$ and $\neg 1 = 0$. \\
$\pi$ & A policy. \\
$\chi$ & An infinite binary string. $\chi_k = 1$ if an agent starts exploring at time-step $k$. \\
$\bar \chi$ & An infinite binary string. $\bar \chi_k = 1$ if an agent is exploring at time-step $k$. \\
$\dot \chi^h$ & An infinite binary string. $\dot \chi^h_k = 0$ if an agent will not explore for the next $h$ time-steps. \\
$\alpha$ & An infinite binary string. \\
$\psi$ & An infinite random binary string sampled from the coin flip measure. \\
$t, n, i, j, k$ & Time indices. \\
$\yx_{<t}$ & A sequence of action/observation/reward sequences. Splits into $y_1o_1r_1 y_2o_2r_2\cdots y_{t-1}o_{t-1}r_{t-1}$.\\
$\yx^{\mu,\pi}_{<t}$ & The sequence of action/reward/observations seen in deterministic environment $\mu$ when playing policy $\pi$. \\
$V^\pi_\mu(\yx_{<t})$ & The value of playing policy $\pi$ in environment $\mu$ starting at history $\yx_{<t}$. \\
$V^*_\mu(\yx_{<t})$ & The value of playing the optimum policy $\pi$ in environment $\mu$ starting at history $\yx_{<t}$. \\
$\pi^*_\mu$ & The optimum policy in environment $\mu$. \\
$H_t(p)$ & The $p$-percentile horizon. \\
$0^{h+1}$ & A binary string consisting of $h+1$ zeros. \\
\hline
\end{tabular}

\end{document}